\documentclass{article}

\usepackage{microtype}
\usepackage{graphicx}
\usepackage{subfigure}
\usepackage{booktabs} % for professional tables

\usepackage{hyperref}
\usepackage{natbib} 

\usepackage{listings}
\usepackage{amsmath}
\usepackage{amssymb}
\usepackage{mathtools}
\usepackage{amsthm}

\let\oldcdot\cdot
 \usepackage{breqn}
\let\cdot\oldcdot

\usepackage[capitalize,noabbrev]{cleveref}

\usepackage{xcolor,amsmath}
\usepackage[linesnumbered,ruled,vlined]{algorithm2e}
\DontPrintSemicolon

\usepackage{enumitem}

\SetKwComment{Comment}{\color{green!50!black}// }{}

\SetKwProg{Function}{function}{}{}

\theoremstyle{plain}
\newtheorem{theorem}{Theorem}[section]
\newtheorem{proposition}[theorem]{Proposition}

\theoremstyle{definition}

\theoremstyle{remark}
\newtheorem{remark}[theorem]{Remark}

\newcommand{\bS}{\mathbf{S}}
\newcommand{\bP}{\mathbb{P}}
\newcommand{\bv}{\mathbf{v}}
\newcommand{\ba}{\mathbf{a}}
\newcommand{\bA}{\mathbf{A}}
\newcommand{\bB}{\mathbf{B}}
\newcommand{\bC}{\mathbf{C}}
\newcommand{\bW}{\mathbf{W}}
\newcommand{\bM}{\mathbf{M}}
\newcommand{\bU}{\mathbf{U}}
\newcommand{\bD}{\mathbf{D}}
\newcommand{\bb}{\mathbf{b}}
\newcommand{\bk}{\mathbf{k}}
\newcommand{\bq}{\mathbf{q}}

\newcommand{\be}{\mathbf{e}}

\newcommand{\bxi}{\boldsymbol{\xi}}
\newcommand{\bomega}{\boldsymbol{\omega}}

\newcommand{\R}{\mathbb{R}}

\newcommand{\sprod}[2]{\langle #1, #2 \rangle}

\begin{document}

\title{ParallelFlow: Parallelizing Linear Transformers via Flow Discretization}

\author{\textbf{Nicola Muca Cirone} \\ Department of Mathematics \\ Imperial College London \\ \emph{n.muca-cirone22@imperial.ac.uk} \and
\textbf{Cristopher Salvi} \\ Department of Mathematics \\ Imperial College London \\ \emph{c.salvi@imperial.ac.uk}}
\date{}

\maketitle

\begin{abstract}
We present a theoretical framework for analyzing linear attention models through matrix-valued state space models (SSMs), building on the foundations established in \cite{cirone2024theoretical}.
Our approach, \emph{Parallel Flows}, provides a perspective that systematically decouples temporal dynamics from implementation constraints, enabling independent analysis of critical algorithmic components: chunking, parallelization, and information aggregation.
Central to this framework is the reinterpretation of chunking procedures \cite{pmlr-v162-hua22a, sun2023retentivenetworksuccessortransformer} as computations of the \emph{flows} governing system dynamics. This connection establishes a bridge to mathematical tools from rough path theory, opening the door to new insights into sequence modeling architectures.
As a concrete application, we analyze DeltaNet \cite{pmlr-v139-schlag21a} in a generalized low-rank setting motivated by recent theoretical advances \cite{siems2025deltaproductimprovingstatetrackinglinear, grazzi2024unlockingstatetrackinglinearrnns}.
Our methods allow us to design simple, streamlined generalizations of hardware-efficient algorithms \cite{yang2024parallelizinglineartransformersdelta} present in the literature, and to provide completely different ones, inspired by rough paths techniques, with provably lower complexity.
This dual contribution demonstrates how principled theoretical analysis can both explain existing practical methods and inspire fundamentally new computational approaches.
\end{abstract}

\section{Introduction}

\subsection{Linear attention transformers as associative linear RNNs} 

The attention mechanism \cite{vaswani2017attention} has become a critical primitive for accurate sequence modeling, offering training efficiency through extensive matrix multiplications that exploit the highly parallel processing capabilities and specialized accelerators of modern GPUs. However, its computational complexity scales quadratically with sequence length, making it inherently expensive for long sequences. Recent techniques \cite{dao2022flashattention,dao2023flashattention} have enabled scaling attention to longer sequences by restructuring intermediate computations in a hardware-aware manner. Despite this progress, these methods still require storing the key and value vectors of previous elements, and managing this "KV cache" can become cumbersome when dealing with very long sequences. \emph{Linear attention transformers} \cite{linear_attention} replace the exponential kernel in softmax attention with a dot product over (possibly transformed) key and query vectors. This adjustment allows linear attention to be formulated as a linear RNN with matrix-valued hidden states, eliminating the need for a KV cache and enabling constant-memory inference. More precisely, \emph{associative linear RNNs} with matrix-valued hidden state $\bS_t \in \mathbb{R}^{d \times d}$ and input sequence $(x_1,...,x_L) \subset \mathbb{R}^d$ are defined by the following equations:
\begin{align}
\bS_t &= \bS_{t-1} \diamond \bA_t + \bv_t \bk_t, \tag{recurrence} \\
\mathbf{o}_t &= \bS_t \bq_t, \tag{memory readout}
\end{align}
where $\diamond$ is an associative product (e.g. Hadamard product, matrix multiplication etc.), and the matrix $\bA_t$ and the vectors $\bv_t, \bk_t, \bq_t$  are (potentially non-linear) functions of the input sequence $x_t$. In this language, the hidden state recurrence in linear attention simply reads as $\bS_t = \bS_{t-1} + \bv_t \bk^\top_t$.

\subsection{Sequential vs parallel form tradeoffs} 
Denoting by $\mathbf{Q}, \mathbf{K}, \mathbf{V} \in \mathbb{R}^{L \times d}$ the matrices of stacked query, key, and value vectors, we can then compute the output $\mathbb{R}^{L \times d} \ni \mathbf{O} = (\mathbf{QK}^\top \odot \mathbf{M}) \mathbf{V}$ in parallel, where $\mathbf{M} \in \mathbb{R}^{L \times L}$ is the standard causal mask.  
This parallel form and the above recurrent form have different FLOPs and memory cost trade-offs. 
The parallel form takes $\mathcal{O}(L^2d + L d^2)$ and thus requires more FLOPs than the recurrent form, which takes $\mathcal{O}(Ld^2)$. 
However, the parallel form can often be faster in practice for moderate-length sequences as it can be fully parallelized and done in $\mathcal{O}(1)$ steps. This sequence-level parallelism also enables high GPU occupancy. While the recurrent form requires fewer FLOPs, the element-wise operations involved in the recurrence have low arithmetic intensity, unlike the matrix-multiplication operations in the parallel form. Thus, the recurrent form is often slower in practice when implemented on GPU hardware. Nonetheless, the presence of an associative operator $\diamond$ enables the use of parallel scan~\cite{BlellochTR90} to parallelize the recurrent form and calculate the hidden states $\bS_1, \dots, \bS_L$ in $\mathcal{O}(\log L)$ steps and $\mathcal{O}(Ld^2)$ FLOPs. However, this approach requires materializing the hidden state $\bS_t$ for each time step, which incurs significant memory I/O cost. Thus, finding an optimal tradeoff between parallelisation and memory footprint is of paramount importance.

\subsection{Gated linear transformers and SSMs}

While early versions of linear attention generally underperformed compared to softmax attention in language modeling tasks, recent ones incorporating data-dependent gating factors have shown competitiveness against strong transformer baselines. These \emph{gated linear transformers} like
\begin{align}
   \bS_t =  &\gamma \bS_{t-1} + \bv_t \bk_t^\top, \tag{RetNet~\cite{qin2023scaling}} \\
  \bS_t = & \bS_{t-1} (\mathbf{1}\alpha_t^\top) + \bv_t \bk_t^\top, \tag{GLA~\cite{yanggated}}  
\end{align}
along with time-varying \emph{state space models} (SSMs) such as
\begin{align}
    \bS_t = & \bS_{t-1} \odot \exp(-(\alpha_t \mathbf{1}^\top) \odot \exp(A)) + (B \odot v_t 1^\top), \tag{S4~\cite{guefficiently}}\\
    \bS_t = & \bS_{t-1} \odot \exp(-(\alpha_t \mathbf{1}^\top) \odot \exp(A)) + (\alpha_t \odot \bv_t)\bk_t^\top, \tag{Mamba~\cite{gu2023mamba}}
\end{align}
have been proposed as potential alternatives to standard transformers. These models make use of cheap element-wise recurrence updates, in particular the Hadamard product, i.e. $\diamond = \odot$.  Standard matrix multiplications (i.e. $\bS_{t-1} \diamond \bA_t = \bS_{t-1} \bA_t$) on the other hand can model richer interactions that go beyond elementwise recurrence. Without any structural assumptions on $\bA_t$ however, these operations would take $\mathcal{O}(dn^2)$ for each update (as opposed to $\mathcal{O}(dn)$ for the Hadamard product), which would be prohibitively expensive. However, despite their competitive language modeling performance, studies have shown that these models underperform compared to transformers on recall-intensive and long context tasks~\cite{arora2024simple}. To improve associative recall over long contexts,~\cite{schlag2020learning} introduced DeltaNet
\begin{equation}
    \bS_t = \bS_{t-1} - \beta_t \bS_{t-1} \bk_t \bk_t^\top + \beta_t \bv_t \bk_t^\top \tag{DeltaNet \cite{schlag2020learning}}
\end{equation}
a variant of linear transformers that utilizes a delta rule-like update to retrieve and update the value vector associated with the current key. DeltaNet’s use of $\bA_t = I - \beta_t \bk_t \bk_t^\top$ (i.e. identity + rank-1 matrix) can be seen as exploiting structured matrices to efficiently model interactions beyond elementwise recurrences. DeltaNet demonstrated effectiveness on synthetic tasks as well as small-scale language modeling and machine translation. 

However, although the original implementation can be parallelized across sequence length by means of parallel scan, as stated above, the naive approach incurs a high memory cost leading to hardware-inefficient training. 
In \cite{yang2024parallelizinglineartransformersdelta} an hardware-efficient algorithm is proposed for rank $1$ updates, finding a good compromise between sequence-length parallelization and memory efficiency.
However, it remains unclear how to efficiently scale DeltaNet to cases with updates with rank $R > 1$.

\subsection{Contributions}

This work strengthens the connection between State Space Models (SSMs) and controlled differential equations (CDEs) established in \cite{cirone2024theoretical}, extending it to the setting of \textit{matrix-valued SSMs} needed to draw connections to linear transformer architectures \cite{Dao_StateSpaceDuality}.

By adopting this high-level perspective, we disentangle \textit{temporal dynamics} from \textit{implementation-specific choices}, enabling a modular analysis of the core components governing numerical performance and algorithm design. Key computational stages--chunking, parallelization, and information aggregation--are systematically decoupled and studied independently. 
Our method, which we dub \emph{Parallel Flows}, rests on interpreting the chunking procedure \cite{pmlr-v162-hua22a, sun2023retentivenetworksuccessortransformer, yang2024parallelizinglineartransformersdelta} as \emph{flow computation}, providing a unified framework for parallelizable pipelines. 
We argue this perspective offers a principled foundation for future investigations into expressivity and scaling limits, see Section \ref{sec:future_work}.

As a first case study, in Section~\ref{subsec:high_rank} we propose a generalized Delta rule \cite{pmlr-v139-schlag21a} for linear transformers, where update matrices are of rank \( R \geq 1 \). This extension is motivated by recent theoretical insights \cite{siems2025deltaproductimprovingstatetrackinglinear,grazzi2024unlockingstatetrackinglinearrnns} demonstrating that increased rank enhances model expressivity.
Placing these models into our framework yields two key contributions:  

\begin{itemize}[topsep=0pt,parsep=0pt] 
    \item \textbf{Hardware-efficient algorithms} (Theorem~\ref{theo:high_rank_delta}): We obtain simplified derivations generalizing to higher ranks existing intra-chunk computation strategies \cite{yang2024parallelizinglineartransformersdelta}. Our solutions address the current lack of dedicated low-rank kernels and, most crucially, it does so by keeping the parallelism across sequence length intact.
    \item \textbf{Signature-Inspired Technique} (Theorem~\ref{theo:sig_delta}): A novel algorithm inspired by signature kernel techniques \cite{salvi2021signature}, computing the same low-rank solution, parallel-in-time, and with theoretically superior temporal scaling properties. A comparison of complexities of the algorithms is presented in Table \ref{tab:comparison}, where it's shown how this new approach improves on the previous solution by one order of magnitude in sequence length $L$.
\end{itemize}

\begin{table*}[t]
\centering
\begin{tabular}{lcc}
\toprule
& \textbf{tensorInv}& \textbf{sigDelta}\\ 
\midrule

Memory & $\mathcal{O}(L^2R^2 + LRd + d^2)$ & $\mathcal{O}(L^2R^2 + LRd + d^2)$     \\
Sequential & $\mathcal{O}(L^2 R (d^2 + Rd + LR^2)))$ & $\mathcal{O}(L^2 R (d^2 + Rd + R)))$ \\
Parallel & $\mathcal{O}(L^2R + d)$  & $\mathcal{O}(LR + d)$                
\end{tabular}
\caption{Comparison of memory and theoretical computational complexities of the proposed algorithms: \textbf{tensorInv} is the algorithm computing $\bS_1$ as per Theorem~\ref{theo:high_rank_delta}, \textbf{sigDelta} as per Theorem~\ref{theo:sig_delta}.}
\label{tab:comparison}
\end{table*}

\section{CDEs: a unifying framework for autoregressive sequence models.}

As originally remarked for SSMs in the recent work~\cite{cirone2024theoretical}, the previous examples are nothing but discretisations of linear \emph{controlled differential equations} (CDEs).
We focus on matrix-valued CDEs of the form
\footnote{Our results extend to more general cases, such as ones where an Hadamard product takes the place of matrix multiplication.}
\begin{equation}
    d\bS_t = \bS_t \cdot d\bomega_t + d\bxi_t, \label{eq:cde}
\end{equation}
driven by matrix-valued paths $\bomega, \bxi: [0, 1] \to \R^{d \times d}$.

% Writing $V \in \mathcal{L}( \R^{d \times d}; \mathcal{L}(\R^{d \times d}; \R^{d \times d}))$ for the linear vector field encoding matrix multiplication \emph{i.e.} $V(A)B := A\cdot B$), we can rewrite the CDE (\ref{eq:cde}) as
% \begin{equation}\label{eq:cde_2}
%     d \bS_t = V(\bS_t) d\bomega_t + \bxi_t 
% \end{equation}

\subsection{Flows as series of iterated integrals}

CDEs are the core object studied in \emph{rough path theory}~\cite{lyons1998differential}, a robust solution theory for non-linear control systems driven by irregular signals.
As proved in \cite{cirone2024theoretical}[Appendix E], the solution to the CDE (\ref{eq:cde}) over an interval $[s,t]$ has the following form:
\begin{equation}
    \bS_t = \mathbb{W}_{s,t} \bS_s + \int_s^t \mathbb{W}_{r,t} d\bxi_r
\end{equation}
% In the case of DeltaNet, these paths are piecewise linear with piecewise constant derivatives on intervals $(t, t+1]$ equal to $-\beta_t k_tk^\top_t$ and $\beta_t v_t k_t^\top$ respectively. 

% Considering the joint path $w_t = (y_t, z_t) \in \mathbb R^{d \times d} \oplus \mathbb R^{d \times d}$ and letting $W = (V,\mathbf{1})$, we can rewrite the CDE (\ref{eq:cde}) in a more compact notation
% \begin{equation}
%     d S_t = W(S_t) dw_t. \label{eq:cde_2}
% \end{equation}
% CDEs are the core object studied in \emph{rough path theory}~\cite{lyons1998differential}, a robust solution theory for non-linear control systems driven by irregular signals. As proved in \cite{cirone2024theoretical}[Appendix E], the solution to the CDE (\ref{eq:cde_2}) over an interval $[s,t]$ has the following form
% \begin{equation}
%     S_t = \mathbb{W}_{s,t} S_s
% \end{equation}
where the \emph{flow} $\mathbb{W}_{s,t}$ is a linear map from $\mathbb{R}^{d \times d}$ to $\mathbb{R}^{d \times d}$ defined as
\begin{align*}
    \mathbb{W}_{s,t} A &= A  \left( \sum_{n=0}^{\infty} ~~ \int\limits_{s < r_1 < \cdots < r_n < t} d\bomega_{r_1} \cdots d\bomega_{r_n} \right) 
\end{align*}

In conclusion we see that:
\begin{proposition}[Flow]\label{prop:flow_body}
    Given matrix-valued paths $\bomega, \bxi: [0, 1] \to \R^{d \times d}$, the CDE defined by 
    \begin{equation}
        d\bS_t = \bS_t \cdot d\bomega_t + d\bxi_t
    \end{equation}
    can be solved on any interval $[s, t] \subseteq [0,1]$ as
    \begin{equation}
    \bS_t = \bS_s\mathbb{P}_{s,t} + \int_s^t d\bxi_r \mathbb{P}_{r,t},
\end{equation}
\begin{equation}
    \mathbb{P}_{s,t} = Id + \int_s^t \mathbb{P}_{s,r} d\bomega_r \in \R^{d \times d}.
\end{equation}
\end{proposition}

The proof of this result is given in Appendix \ref{app:proofs}.

\section{Parallel Flows}
\label{sec:Parallel_Flows}

We now build on the general CDE formalism by proposing the \emph{Parallel Flow} perspective, which-by interpreting chunking procedures as computations of the \emph{flows} of Proposition~\ref{prop:flow_body}-enables independent analysis of critical algorithmic components without being tied to specific architectures. As established in Proposition~\ref{prop:flow_body} the flow can be expressed as:
\begin{equation*}
    \bS_t = \bS_s\mathbb{P}_{s,t} + \int_s^t d\bxi_r \mathbb{P}_{r,t},
\end{equation*}
\begin{equation*}
    \mathbb{P}_{s,t} = Id + \int_s^t \mathbb{P}_{s,r} d\bomega_r \in \R^{d \times d}.
\end{equation*}

Given the state $\bS_s$ at time $s$, this formulation enables direct computation of $\bS_t$ for any $t > s$ without requiring intermediate states $\bS_r$ for $r \in (s, t)$, provided one has precomputed the propagator matrix $\mathbb{P}_{s,t}$ and the integrated term $\int_s^t d\bxi_r \mathbb{P}_{r,t}$. 
This perspective unlocks an efficient parallelization strategy for computing $\bS_1$:
\begin{enumerate}
    \item \textbf{Chunk.} Partition $[0,1]$ into $M$ sub-intervals $\{(t_{k-1}, t_k]\}_{k=1}^L$ using grid points $0 = t_0 < \cdots < t_L = 1$.
    
    \item \textbf{Parallel Compute.} Independently compute $\mathbb{P}_{t_{k-1},t_k}$ and $\int_{t_{k-1}}^{t_k} d\bxi_r \mathbb{P}_{r,t_k}$ for each sub-interval. These computations are fully parallel as they involve no cross-interval dependencies.
    
    \item \textbf{Scan.} Aggregate results across intervals using a parallel scan algorithm \cite{smith2023simplified}, propagating the initial condition $\bS_0$ through the composed flow.
\end{enumerate}

\noindent\textbf{Interval Selection.} The choice of partition points $\{t_m\}$ impacts both computational efficiency and numerical stability. While uniform partitions suffice for basic implementations, adaptive step-size strategies based on regularity properties of the driving path $\omega$ could optimize chunk granularity. Another interesting approach is provided by the \emph{log-ODE method}, a higher-order solver which allows to take larger steps at the cost of using higher-order iterated integrals of the path $\omega$ captured by its \emph{log-signature}. See for example \citep[Section 3.2.2]{cass2024lecture} for a description of the algorithm and \citep{morrill2021neural} for a description of how to integrate it within a continuous-time sequence model. Quantifying the trade-offs between approximation quality and computational gains through more refined adaptive partitioning or higher-order mechanism remains an open question for future work.

\vspace{0.2cm}
\noindent\textbf{Intra-Chunk Computation.}
Efficient intra-chunk computation of $\mathbb{P}_{t_{k-1},t_k}$ and $\int_{t_{k-1}}^{t_k} d\bxi_r \mathbb{P}_{r,t_k}$ becomes crucial for practical efficiency. By the parallel nature of the problem, one may assume without loss of generality that $(t_{k-1}, t_k] \equiv (0,1]$ through temporal rescaling. Here, three key factors dominate implementation: (1) the algebraic structure of the driving processes $\bomega$ and $\bxi$, (2) the path representation choice (\emph{e.g.}, piecewise-linear, polynomial, or piecewise-constant approximations), and (3) the numerical methods employed for evaluating the flow operators. These design decisions ultimately determine the computational complexity-accuracy tradeoff in implementations, and as such are of critical importance. In particular one should make sure the structure of the objects allows the computation of the flow operators through the highly efficient tensor operations of modern hardware. In this paper we will focus mainly on 1) and 3) and we leave the choice of path interpolator to future work.

\subsection{Low-Rank Delta Rule}
\label{subsec:high_rank}

The parallel flow formulation motivates architectural choices for the driving processes $\bxi$ and $\bomega$ that balance expressiveness with computational efficiency. We present a concrete realization through \emph{Low-Rank Delta Rules}, demonstrating how low-rank matrix parameterizations enable both theoretically richer representations, as recently proved by \citet{grazzi2024unlockingstatetrackinglinearrnns}, and inherit the hardware-friendly implementations of the rank-1 case of \citet{yang2024parallelizinglineartransformersdelta}.

Under this light we propose to address the first two key factors determining the performance of intra-chunk computation as follows :
\begin{itemize}
    \item \textbf{Driver Parameterization:} For matrix-valued paths $\bA, \tilde{\bA}, \bB: [0,1] \to \R^{d \times R}$, we define the drivers as low-rank ($R$) temporal interactions:
    \begin{equation}
        d\bomega_t = \bA_t \bB^\top_t dt \in \R^{d \times d}, \quad
        d\bxi_t = \tilde{\bA}_t \bB^\top_t dt \in \R^{d \times d}
    \end{equation}
    Crucially $\bB$ is shared by both drivers, this will allow for an efficiently computable, compact representation of the flow as shown in Proposition \ref{prop:continuous_AB_flow}.
    
    \item \textbf{Discretization Scheme:} We consider piecewise-linear drivers on a grid $0 = t_0 < \cdots < t_L = 1$. This induces discrete updates (with $\Delta t$ terms absorbed into $\bB$), where $0 \leq k < L$:
    \begin{equation}\label{eqn:discretization}
        \Delta \bomega_{t_k} = \bA_{t_k} \bB^\top_{t_k}, \quad
        \Delta \bxi_{t_k} = \tilde{\bA}_{t_k} \bB^\top_{t_k}
    \end{equation}
\end{itemize}

\begin{remark}
This constitutes a higher-rank generalization of the DeltaNet architecture from \cite{yang2024parallelizinglineartransformersdelta}, where the original rank-1 formulation corresponds to the special case $R=1$ with $\bA_t, \tilde{\bA}_t \in \R^{d}$ and $\bB_t \in \R^{d}$ given by 
\begin{equation}
    \bA_t = \beta_t \bk_t, 
    \quad
    \tilde{\bA}_t = -\beta_t \bv_t, 
    \quad 
    \bB_t = - \bk_t
\end{equation}
with both $\beta_t: [0,1] \to \R$, and  $\bk, \bv: [0,1] \to \R^d$ data-dependent paths.
\end{remark}

The following result, proved in Appendix~\ref{app:parallelFlows_proofs}, presents in continuous time a useful decomposition of the flow:
\begin{proposition}
\label{prop:continuous_AB_flow}
    For matrix-valued paths $\bA, \tilde \bA,  \bB: [0,1] \to \R^{d \times R}$ let the drivers be given by
\begin{equation}
    d \bomega_t =  \bA_t \bB^\top_t dt \in \mathbb{R}^{d \times d},
    \quad
    d \bxi_t =  \tilde\bA_t \bB^\top_t dt \in \mathbb{R}^{d \times d}
\end{equation}
then the flow can be written as 
\begin{equation}\label{eqn:HR_DeltaRule_flow}
    \bS_t 
    %= \bS_s + \bS_s  \int_s^t  \bW_{s,r}\bB_r^\top dr + \int_s^t \bU_{s,r} \bB_r^\top dr
    = \bS_s + \int_s^t \left( \bS_s \bW_{s,r} + \bU_{s,r} \right) \bB_r^\top dr
\end{equation}
where $\bW_{s,t}$ and $\bU_{s,t}$ are the solution of the following integral equations
\begin{equation}
    \bW_{s,t} =  \bA_t + \int_s^t \bW_{s,r} \bB^\top_r \bA_t ~dr,
\end{equation}
\begin{equation}
    \bU_{s,t} =
    \tilde\bA_t + \int_s^t \bU_{s, r} \bB_r^\top \bA_t dr.
\end{equation}
\end{proposition}

\paragraph{Discretization and Intra-Chunk computation}
As previously discussed, thanks to the parallel nature of the problem and through temporal rescaling one may assume without loss of generality the chunck to be the entire interval $(0,1]$ where, with a slight abuse of notation by writing $\bW_{t} := \bW_{0,t}$ and $\bU_{t} := \bU_{0,t}$, we have
\begin{equation}
\label{eqn:chunk_discrete_flow_1}
    \bS_1
    %= \bS_s + \bS_s  \int_s^t  \bW_{s,r}\bB_r^\top dr + \int_s^t \bU_{s,r} \bB_r^\top dr
    = \bS_0 + \int_0^1 \left( \bS_0 \bW_{r} + \bU_{r} \right) \bB_r^\top dr \in \R^{d \times d}
\end{equation}
\begin{equation}
\label{eqn:chunk_discrete_flow_2}
    \bW_{t} =  \bA_t + \int_0^t \bW_{r} \bB^\top_r \bA_t ~dr \in \R^{d \times R},
\end{equation}
\begin{equation*}
    \bU_{t} =
    \tilde\bA_t + \int_0^t \bU_{r} \bB_r^\top \bA_t dr \in \R^{d \times R}.
\end{equation*}

Under discretization (\ref{eqn:discretization}), on the grid $0 = t_0 < \cdots < t_L = 1$, the dynamics can be discretised as 
\begin{equation}
    \bS_{t_{k+1}} = \bS_{t_k} + \bS_{t_k} \cdot \bA_{t_k}\bB_{t_{k}}^\top +  \tilde \bA_{t_k}\bB_{t_{k}}^\top.
\end{equation}
and the \emph{flow} equations (\ref{eqn:chunk_discrete_flow_1}) and (\ref{eqn:chunk_discrete_flow_2}) become
\begin{equation}
    \bS_{1} = \bS_{0} 
    + \sum_{k = 0}^{L-1} 
    \left( \bS_{0}  \bW_{t_k} + \bU_{t_k} \right) \bB^\top_{t_k}
\end{equation}
\begin{equation}
\label{equn:discrete_W}
    \bW_{t_{k}} =  \bA_{t_{k}} + \sum_{m = 0}^{k-1} \bW_{t_m} \bB^\top_{t_{m}} \bA_{t_k},
\end{equation}
\begin{equation*}
    \bU_{t_{k}} =
    \tilde\bA_{t_k} + \sum_{m = 0}^{k-1} \bU_{t_m} \bB_{t_m}^\top \bA_{t_k}.
\end{equation*}

These discretized equations can be completely written in terms of highly-efficient tensor operations:
\begin{theorem}
\label{theo:high_rank_delta}
    Define, with a slight abuse of notation, the tensors $\bA, \tilde\bA, \bB, \bW, \bU \in \R^{(L \times R) \times d}$ defined from the above equations and with entries $[ \square ]_{k, i}^{m} = [ \square_{t_{k-1}} ]_{m}^i$. Then
    \begin{equation}
        \bS_1 = \bS_{0} + (\bU + \bW \bS_0^\top)^\top \bB
    \end{equation}
    \begin{equation*}
        \bW = \bA + (\bM \odot \bA\bB^\top) \bW, \quad \bU = \tilde\bA + (\mathbf{M} \odot \bA \bB^\top) \bU
    \end{equation*}
    Where $\bM \in \R^{(L \times R) \times (L \times R)}$ has entries $[\bM]_{t, i}^{s, j} = \mathbb{I}(s < t)$ and composition of tensors corresponds to contraction of corresponding co-variant and contra-variant indices \emph{e.g.}
    \[
     [\bA\bB^\top]_{k, i}^{k', i'}
     := \sum_{m=1}^d [\bA]_{k, i}^{m} [\bB^\top]_m^{k', i'}
     = \sum_{m=1}^d [\bA]_{k, i}^{m} [\bB]^m_{k', i'}
    \]
    
    Here $\odot$ stads for component-wise product. 
    In particular the \emph{implicit} equations defining $\bW$ and $\bU$ can be \emph{explicitly} solved as 
    \begin{equation}
        \bW = (\mathbf{Id} - \bM \odot \bA\bB^\top)^{-1} \bA,
    \end{equation}
    \begin{equation}
        \bU = (\mathbf{Id} - \bM \odot \bA\bB^\top)^{-1} \tilde\bA
    \end{equation}
    where $[\mathbf{Id}]_{s, i}^{t, j} = \delta_s^t \delta_i^j$ and $(\mathbf{Id} - \bM \odot \bA\bB^\top)^{-1}$ is the inverse of $(\mathbf{Id} - \bM \odot \bA\bB^\top) \in \R^{(L \times R) \times (L \times R)}$.
\end{theorem}

Note that the triangular nature of the matrix to invert allows for an efficient \emph{forward-substitution} algorithm,  which can be computed sequentially in $\mathcal{O}(L^3R^3)$ steps, but importantly also in parallel with lowered $\mathcal{O}(L^2R)$ complexity:
\begin{proposition}[Triangular Tensor Inversion]
    Let $\bC \in \R^{(T \times R) \times (T \times R)}$ be such that 
    \begin{equation*}
        s > t \implies [ \bC ]_{t, i}^{s, j} = 0, \quad \text{and } \
         [ \bC ]_{t, \cdot}^{t, \cdot} \in GL(R), \  \forall t.
    \end{equation*}
    Define the tensor $\mathbf{D} \in \R^{(T \times R) \times (T \times R)}$ as follows:
    
    for $s > t$
    \begin{equation*}
        \mathbf{D}_{t, \cdot}^{s, \cdot} = \mathbf{0} \in \R^{R \times R},
    \end{equation*}
    for $s=t$
    \begin{equation*}
        \mathbf{D}_{t, \cdot}^{t, \cdot} = ( \mathbf{C}_{t, \cdot}^{t, \cdot} )^{-1} \in \R^{R \times R},
    \end{equation*}
    and for $s < t$
    \begin{align*}
        \mathbf{D}_{t, \cdot}^{s, \cdot} &= - 
        ( \mathbf{C}_{t, \cdot}^{t, \cdot} )^{-1} 
        \left( 
        \sum_{r=s}^{t-1} 
        \mathbf{C}_{t, \cdot}^{r, \cdot}
        \mathbf{D}_{r, \cdot}^{s, \cdot} 
        \right) \\
        &= - 
        \bD_{t, \cdot}^{t, \cdot}
        \left( 
        \sum_{r=s}^{t-1} 
        \mathbf{C}_{t, \cdot}^{r, \cdot}
        \mathbf{D}_{r, \cdot}^{s, \cdot} 
        \right),
    \end{align*}
    Then, $\mathbf{D}$ is the the inverse of $\bC$ for tensor composition, \emph{i.e.} $$\bD \bC = \mathbf{Id}.$$
\end{proposition}

Proofs for the results above can be found in Appendix~\ref{app:parallelFlows_proofs}.

Figure~\ref{code_tensorInv} presents our PyTorch \cite{paszke2017automatic} implementation for computing $\bS_1$. As summarized in Table~\ref{tab:comparison}, the algorithm requires $\mathcal{O}(L^2R^2 + LRd + d^2)$ memory complexity. The computational complexity is $\mathcal{O}(L^2 R (d^2 + Rd + LR^2)))$ in sequential execution, which can theoretically be further reduced to $\mathcal{O}(L^2R + d)$ via parallel computation.

\subsection{A signature kernel inspired algorithm}

We introduce a new algorithm, drawing inspiration from signature kernel techniques \cite{salvi2021signature}, that does not require an explicit tensor inversion, and thus exhibits theoretically superior scaling behavior with $\mathcal{O}(L^2 R (d^2 + Rd + R)))$ sequential and $\mathcal{O}(LR + d)))$ parallel complexities, furthermore maintaining the same $\mathcal{O}(L^2R^2 + LRd + d^2)$ memory footprint (see Table \ref{tab:comparison}). 
A code snippet of a PyTorch implementation is given in Figure \ref{sigDelta_code_optimal}.

\begin{theorem}
\label{theo:sig_delta}
    The solution to the implicit equation 
    \begin{equation}
        \bW_{t_{k}} =  \bA_{t_{k}} + \sum_{m = 0}^{k-1} \bW_{t_m} \bB^\top_{t_{m}} \bA_{t_k},
    \end{equation}
    is given by the diagonal elements $\bW(t_k, t_k)$ of the system
    defined by
    \begin{equation*}
        \bW(t_0, t_k) = \bA_{t_k}, 
    \end{equation*}
    \begin{equation*}
        \quad \bW(t_{k+1},t_{k+1}) = \bW(t_k,t_{k+1}) + \bW(t_k,t_k) \bB_{t_k}^\top \bA_{t_{k+1}}
    \end{equation*}
    and for $m < k$
    \begin{dmath*}
        \bW(t_{m+1}, t_{k+1}) =  \bW(t_{m},t_{k+1}) + \bW(t_{m+1},t_{k}) - \bW(t_{m},t_{k}) + \bW(t_{m},t_{m}) 
        \bB^\top_{t_{m}} ( \bA_{t_{k+1}} - \bA_{t_k}).
    \end{dmath*}
\end{theorem}

Refer to Appendix~\ref{app:sig_algo_proofs} for the proof.

The system described in Theorem \ref{theo:sig_delta} closely resembles an equation that defines a key concept in rough path theory, the \emph{signature kernel}, for which a similar parallelisable algorithm was provided by \citet{salvi2021signature}. We note that this kernel has also been identified as the scaling limit of a large class of SSMs with Gaussian vector fields by \citet{cirone2023neural}.

Building on this analogy, we can achieve parallelization of the dynamics by reordering the computation. Rather than processing the equation row-by-row or column-by-column, we update the solution grid along its “antidiagonals.” 
Since there are no data dependencies within an antidiagonal, all its elements can be updated simultaneously.
This breaks the quadratic complexity in $L$, that becomes $\mathcal{O}(L)$, thus achieving linear memory and linear time complexity for the computation of the dynamics.

However, we observe a gap between theoretical expectations and our practical implementation: the expected complexity improvements do not materialize due to technical constraints imposed by Triton, our chosen GPU kernel framework. These limitations arise from fundamental restrictions in tensor operation support within modern accelerator programming paradigms, as we will explain in more details in the experimental section \ref{sec:experiments}.

\subsection{An alternative representation of the flow as product of exponentials} 

We note that on a discrete grid $0 = t_0 < \cdots < t_L = 1$ the following representation of the flow holds
\begin{equation}
\label{discrete_flow_prod}
    \mathbb{P}_{t_k,t_{m}} = \prod_{i=k}^{m-1}\exp(\Delta \bomega_{t_i}) \in \mathbb{R}^{d\times d}
\end{equation}
where $\Delta \bomega_{t_k} = \bomega_{t_{k+1}} - \bomega_{t_k}$ is an increment.

The discretization of the flow $\bP$ via Euler-Scheme, for example used in \cref{equn:discrete_W}, corresponds to a naive approximation 
\[
\exp(\Delta\bomega) = Id + \Delta\bomega
\]
in \cref{discrete_flow_prod}.
Note however how, for $\bA, \bB \in \R^{d \times R}$, one has
\begin{dmath*}
    \exp(\bA \bB^\top) 
    = Id + \sum_{k=1}^{\infty} \frac{1}{k!} (\bA \bB^\top)^k
    = Id + \bA \left( \sum_{k=1}^{\infty} \frac{1}{k!} (\bB^\top \bA)^{k-1} \right) \bB^\top
    = Id + \bA (\bB^\top\bA)^{-1} \left( \exp(\bB^\top \bA) - Id \right) \bB
\end{dmath*}
where $\bB^\top\bA \in \R^{R \times R}$.  If the matrices are rank-1 the previous computations simplify dramatically. In effect, when $R=1$ we can compute the matrix exponential, without explicitly matrix exponentiations, as 
\begin{dmath*}
    \exp(\ba \bb^\top) = Id + \left( \frac{e^{\sprod{\bb}{\ba} - 1}}{\sprod{\bb}{\ba}} \right) \ba \bb^\top
    = Id + \left( \frac{e^{tr(\ba \bb^\top) - 1}}{tr(\ba \bb^\top)} \right) \ba \bb^\top
\end{dmath*}

Thus in case $\Delta\bomega_{t_k}$ is always of rank 1, one ends up with the \emph{exact} formula
\begin{equation}
    \mathbb{P}_{t_k,t_{m}} = 
    \prod_{i=k}^{m-1}\left( Id - \left( \frac{e^{tr(\bomega_{t_i}) - 1}}{tr(\bomega_{t_i})} \right) \bomega_{t_i} \right).
\end{equation}

\section{Numerics}
\label{sec:experiments}

While Theorem~\ref{theo:sig_delta} demonstrates an $L$-factor complexity improvement over the tensor-inversion approach of Theorem~\ref{theo:high_rank_delta}, in both sequential and parallel regimes, practical realization of these gains remains elusive.
This discrepancy stems from fundamental constraints in the Triton programming framework \cite{triton_paper}, which we adopted to ensure fair comparison with the original DeltaNet implementation \cite{yang2024parallelizinglineartransformersdelta, yang2024fla} on top of which we built the low-rank, tensor-inversion extension.

The crux of the implementation bottleneck lies in Triton's lack of native support for tensor-slicing operations - a requirement for two critical components:
\begin{itemize}[topsep=0pt,parsep=0pt]
    \item efficient parallel computation of future values $\square_{t_{k+1}}$ (implemented via \texttt{torch.roll} operations in the code snippet of Figure~\ref{sigDelta_code_optimal}),
    \item anti-diagonal extraction from the tensor $[\bB_{t_m}^\top (\bA_{t_{k+1}} - \bA_{k})]_{m,k}$, part of the core mechanism enabling the complexity-breaking mechanism.
\end{itemize}

These limitations fundamentally constrain our ability to exploit the theoretical advantages of the signature-inspired approach, despite its mathematical promise. 
The challenge highlights an important practical consideration when bridging theoretical complexity analysis with hardware-aware implementations: framework-specific constraints can dominate algorithmic improvements at scale.

We provide both PyTorch and Triton kernel implementations in the \emph{supplementary material}.

These implementation challenges suggest concrete engineering pathways: careful memory management through CUDA-level kernel optimization could circumvent current Triton limitations. We hypothesize that direct control over GPU memory would unlock the theoretical $\mathcal{O}(L)$ advantage demonstrated in Table~\ref{tab:comparison}, leaving this empirical validation for future work.

\section{Conclusions and Future Directions}
\label{sec:future_work}

The \emph{Parallel Flows} framework we introduced in this paper establishes a principled foundation for analyzing matrix-valued SSMs through the lens of controlled differential equations. We made the following key contributions: (1) provided a simplified derivations of existing hardware-efficient algorithms, (2) generalized to the more expressive low-rank regime, while preserving parallelism across sequence length, and (3) discovered fundamentally new, parallel in time, computational approaches inspired by rough path theory leading to an improvement in computational cost by an entire order of magnitude in sequence length $L$.

Beyond our current results, this perspective opens three promising research directions:

\begin{itemize}
    \item \textbf{Scaling limit analysis}: Leveraging the CDE formalism to derive closed-form solutions for infinite-depth limits, following the methodology of \cite{cirone2023neural, cirone2024theoretical}. This would enable precise characterization of expressive power.
    \item \textbf{Adaptive step-size solvers} As mentioned before, while uniform partitions are adequate for basic implementations, adaptive step-size strategies that leverage the regularity properties of the driving path $\omega$ can optimize chunk granularity more effectively. 
    \item \textbf{Higher order solvers} Higher order CDE solvers, such as the log-ODE method provides a particularly promising approach, offering a solution that accommodates larger step sizes by utilizing higher-order iterated integrals of the driving path $\omega$ through its log-signature. This method has the potential to reduce computational complexity by a factor of $1/k$, where $k$ represents the number of steps captured by the log-signature locally, aligning with the fundamental principles of rough path theory.
    
\end{itemize}

% \section*{Acknowledgements}

\newpage

\bibliography{bibliography}
\bibliographystyle{abbrvnat}

%%%%%%%%%%%%%%%%%%%%%%%%%%%%%%%%%%%%%%%%%%%%%%%%%%%%%%%%%%%%%%%%%%%%%%%%%%%%%%%
%%%%%%%%%%%%%%%%%%%%%%%%%%%%%%%%%%%%%%%%%%%%%%%%%%%%%%%%%%%%%%%%%%%%%%%%%%%%%%%
% APPENDIX
%%%%%%%%%%%%%%%%%%%%%%%%%%%%%%%%%%%%%%%%%%%%%%%%%%%%%%%%%%%%%%%%%%%%%%%%%%%%%%%
%%%%%%%%%%%%%%%%%%%%%%%%%%%%%%%%%%%%%%%%%%%%%%%%%%%%%%%%%%%%%%%%%%%%%%%%%%%%%%%
\newpage
\appendix
\onecolumn

\section{Theoretical Proofs}
\label{app:proofs}

\subsection{Flow Equation for matrix-valued SSMs}

\begin{proposition}
    Given matrix valued paths $\bomega, \bxi: [0, 1] \to \R^{d \times d}$, the CDE defined by 
    \begin{equation}
        d\bS_t = \bS_t \cdot d\bomega_t + d\bxi_t
    \end{equation}
    can be solved on any interval $[s, t] \subseteq [0,1]$ as
    \begin{equation}
    \bS_t = \bS_s\mathbb{P}_{s,t} + \int_s^t d\bxi_r \mathbb{P}_{r,t}, \quad \mathbb{P}_{s,t} = Id + \int_s^t \mathbb{P}_{s,r} d\bomega_r \in \R^{d \times d}.
\end{equation}
\end{proposition}

\begin{proof}
Here $V \in \mathcal{L}( \R^{d \times d}; \mathcal{L}(\R^{d \times d}; \R^{d \times d}))$ is the linear vector field encoding matrix multiplication \emph{i.e.} $V(A) \cdot B := A \cdot B$).

Note how we can write the field action in coordinates as 
\[
V(A)\cdot B = A \cdot B = \sum_{i, j = 1}^d  (A \cdot \be_i \otimes \be_j ) ~ B_{i,j} = \sum_{i, j = 1}^d  V_{i,j}(A) ~ B_{i,j}.
\]

But we know (\cite{cirone2024theoretical} Appendix E) how to write the explicit solution to a CDE of this form:
\begin{equation}
    \bS_t = \mathbb{W}_{s,t} \bS_s + \int_s^t \mathbb{W}_{r,t} d\bxi_r
\end{equation}
where $\mathbb{W} := W^{V, \bomega}: \{(s,t) \in [0, T]\times[0,T] ~|~ s \leq t\} \to \mathcal{L}(\R^{d\times d}; \R^{d \times d})$ is the \emph{Wronskian} matrix defined by the series
\begin{align*}
     W^{V, \bomega}_{s,t}(\bullet) :&= \sum_{n=0}^{\infty} V^{\otimes n}(\bullet) Sig(\bomega)^{\otimes n}_{s,t}
     \\
     &= \sum_{(I,J)} V_{i_n, j_n} \circ \cdots \circ V_{i_1, j_1} (\bullet) Sig(\bomega)^{(I,J)}_{s,t}.
\end{align*}    

Note that since $V_{i,j}(A) = A \cdot \be_i \otimes \be_j$ then 
\begin{align*}
   V_{i_n, j_n} \circ \cdots \circ V_{i_1, j_1} (A) &= 
   A \cdot (\be_{i_1} \otimes \be_{j_1}) \cdots (\be_{i_n} \otimes \be_{j_n})\\
    &= A \cdot (\be_{i_1} \otimes \be_{j_n}) ~ \prod_{k=2}^{n} \delta_{j_{k-1}, i_k}
\end{align*}
hence we obtain
\begin{align*}
    V^{\otimes n}(A) Sig(\bomega)^{\otimes n}_{s,t} &= 
    \sum_{i,j =1}^d \sum_{\alpha_2, \dots, \alpha_n = 1}^d A \cdot (\be_{i_1} \otimes \be_{j}) ~ 
    \iiint\limits_{s < r_1 < \cdots < r_n < t} d\bomega^{i,\alpha_2}_{r_1} \cdots d\bomega^{\alpha_n,j}_{r_n} \\
    &= 
    \sum_{i,j =1}^d A \cdot (\be_{i_1} \otimes \be_{j}) ~ 
    \iiint\limits_{s < r_1 < \cdots < r_n < t} [ d\bomega_{r_1} \cdots d\bomega_{r_n}]_{i,j} \\
    &= 
    A \cdot 
    \iiint\limits_{s < \cdots < t} d\bomega_{r_1} \cdots d\bomega_{r_n} 
\end{align*}
from which
\begin{equation}
    \mathbb{W}_{s,t}(A) = A \cdot \left( \sum_{n=0}^{\infty} ~~ \iiint\limits_{s < \cdots < t} d\bomega_{r_1} \cdots d\bomega_{r_n} \right)
\end{equation}

Let 
\begin{equation*}
    \bP_{s,t} = \sum_{n=0}^{\infty} ~~ \iiint\limits_{s < \cdots < t} d\bomega_{r_1} \cdots d\bomega_{r_n}
\end{equation*}
then 
\begin{equation*}
    \bP_{s,t} = Id + \sum_{n=1}^{\infty} ~~ \iiint\limits_{s < \cdots < t} d\bomega_{r_1} \cdots d\bomega_{r_n} = Id + \int_s^t \mathbb{P}_{s,r} d\bomega_r
\end{equation*}
and as needed
\begin{equation*}
    \bS_t = \bS_s\mathbb{P}_{s,t} + \int_s^t d\bxi_r \mathbb{P}_{r,t}.
\end{equation*}

\end{proof}

\subsection{Parallel Flows}

\subsubsection{Low-Rank Delta Rule}
\label{app:parallelFlows_proofs}

\begin{proposition}
    For matrix-valued paths $\bA, \tilde \bA,  \bB: [0,1] \to \R^{d \times R}$ let the drivers be given by
\begin{equation}
    d \bomega_t =  \bA_t \bB^\top_t dt \in \mathbb{R}^{d \times d},
    \quad
    d \bxi_t =  \tilde\bA_t \bB^\top_t dt \in \mathbb{R}^{d \times d}
\end{equation}
then the flow can be written as 
\begin{equation}\label{eqn:HR_DeltaRule_flow_app}
    \bS_t 
    %= \bS_s + \bS_s  \int_s^t  \bW_{s,r}\bB_r^\top dr + \int_s^t \bU_{s,r} \bB_r^\top dr
    = \bS_s + \int_s^t \left( \bS_s \bW_{s,r} + \bU_{s,r} \right) \bB_r^\top dr
\end{equation}
where $\bW_{s,t}$ and $\bU_{s,t}$ are the solution of the following integral equations
\begin{equation}
    \bW_{s,t} =  \bA_t + \int_s^t \bW_{s,r} \bB^\top_r \bA_t ~dr,
    \quad
    \bU_{s,t} =
    \tilde\bA_t + \int_s^t \bU_{s, r} \bB_r^\top \bA_t dr.
\end{equation}
\end{proposition}

\begin{proof}
    It suffices to show that
    \[
    \bS_s\bP_{s,t} = \bS_s + \int_s^t \bS_s \bW_{s,r} \bB^\top_r ~ dr, 
    \quad
    \int_s^t d\bxi_r \bP_{s,r} = \int_s^t \bU_{s,r}\bB^\top_r ~ dr.
    \]

    To see this note how
    \begin{align}
    \mathbb{P}_{s,t} &= Id + \int_s^t \mathbb{P}_{s,r} d\bomega_r
    = Id + \int_s^t \mathbb{P}_{s,r} \bA_r \bB^\top_r dr
    \end{align}
    and defining $\bW_{s,t} := \mathbb{P}_{s,t} \bA_t$ we obtain 
    \begin{align}
        \bW_{s,t} = \mathbb{P}_{s,t} \bA_t = \bA_t + \int_s^t \mathbb{P}_{s,r} \bA_r \bB^\top_r \bA_t ~ dr =  \bA_t + \int_s^t \bW_{s,r} \bB^\top_r \bA_t ~dr
    \end{align}

    Regarding the term $ \int_s^t d\bxi_r \mathbb{P}_{r,t}$ note
    \begin{equation}
         \int_s^t d\bxi_r \mathbb{P}_{r,t} = 
         \int_s^t  \tilde\bA_r \bB^\top_r \mathbb{P}_{r,t} dr =
         \int_s^t  \tilde\bA_r \bB^\top_r dr + \int_{r=s}^t \int_{u=r}^t \tilde\bA_r \bB^\top_r \mathbb{P}_{r,u} \bA_u \bB^\top_u du dr
    \end{equation}
    
    Now exchanging the order of the integrals we obtain
    \begin{align*}
        \int_{r=s}^t \int_{u=r}^t \tilde\bA_r \bB^\top_r \mathbb{P}_{r,u} \bA_u \bB^\top_u du
        &= 
        \int_{u=s}^t \int_{r=s}^u \tilde\bA_r \bB^\top_r \mathbb{P}_{r,u} \bA_u \bB^\top_u dr du
        \\&= 
        \int_{r=s}^t \int_{u=s}^r \tilde\bA_u \bB^\top_u \mathbb{P}_{u,r} \bA_r \bB^\top_r du dr
    \end{align*}
    so that, and here crucially the fact that $\bB$ is the same in both $\bxi$ and $\bomega$, one has
    \begin{equation}
        \int_s^t d\bxi_r \mathbb{P}_{r,t} = 
        \int_s^t \left(
        \tilde\bA_r + \int_{u=s}^r \tilde\bA_u \bB^\top_u \mathbb{P}_{u,r} \bA_r du
        \right) \bB^\top_r dr.
    \end{equation}
    
    Let $\bU_{s,t} := \tilde\bA_t + \int_{r=s}^t \tilde\bA_r \bB^\top_r \mathbb{P}_{r,t} \bA_t dr$ then
    \begin{align}
        \bU_{s,t} &= \tilde\bA_t + \int_{r=s}^t \tilde\bA_r \bB^\top_r \mathbb{P}_{r,t} \bA_t dr
        = \tilde\bA_t + \int_s^t d\bxi_r \mathbb{P}_{r,t}\bA_t
        =
        \tilde\bA_t + \int_s^t \bU_{s, r} \bB_r^\top \bA_t dr.
    \end{align}
    and just as wanted
    \[
    \int_s^t d\bxi_r \bP_{s,r} = \int_s^t \bU_{s,r}\bB^\top_r ~ dr.
    \]
    
\end{proof}

\paragraph{Discretization and Intra-Chunk computation}

\begin{theorem}
    Define, with abuse of notation, the 3-tensors $\bA, \tilde\bA, \bB, \bW, \bU \in \R^{(L \times R) \times d}$ with entries $[ \square ]_{k, i}^{m} = [ \square_{t_{k-1}} ]_{m}^i$ then
    \begin{equation}
        \bS_1 = \bS_{0} + (\bU + \bW \bS_0^\top)^\top \bB
    \end{equation}
    \begin{equation}
        \bW = \bA + (\bM \odot \bA\bB^\top) \bW, \quad \bU = \tilde\bA + (\mathbf{M} \odot \bA \bB^\top) \bU
    \end{equation}
    Where $\bM \in \R^{(L \times R) \times (L \times R)}$ has entries $[\bM]_{t, i}^{s, j} = \mathbb{I}(s < t)$ and composition of tensors corresponds to contraction of corresponding co-variant and contra-variant indices. 
    In particular the \emph{implicit} equations defining $\bW$ and $\bU$ can be \emph{explicitly} solved as 
    \begin{equation}
        \bW = (\mathbf{Id} - \bM \odot \bA\bB^\top)^{-1} \bA, \quad 
        \bU = (\mathbf{Id} - \bM \odot \bA\bB^\top)^{-1} \tilde\bA
    \end{equation}
    where $[\mathbf{Id}]_{s, i}^{t, j} = \delta_s^t \delta_i^j$ and $(\mathbf{Id} - \bM \odot \bA\bB^\top)^{-1}$ is the inverse of $(\mathbf{Id} - \bM \odot \bA\bB^\top) \in \R^{(L \times R) \times (L \times R)}$.
\end{theorem}

\begin{remark}
    The notation employs an axis swap, defined by the relation $[\square]_{k, \cdot}^\cdot = \square_{t_{k-1}}^\top$ for $\bA, \tilde\bA, \bB, \bW, \bU$.
    This choice aligns mathematical conventions, where elements of $\R^d$ are naturally treated as column vectors, with practical implementations, where data is often stored and manipulated as row vectors.
\end{remark}

\begin{proof}
    This is just a matter of checking that entries correspond. Writing for 3-tensors $[\square^\top]_m^{k,i} := [\square]_{k,i}^m $ we have 
    \begin{align*}
        \bS_{1} &= \bS_{0} 
    + \sum_{k = 0}^{L-1} 
    \left( \bS_{0}  \bW_{t_k} + \bU_{t_k} \right) \bB^\top_{t_k}
    = \bS_{0} 
    + \sum_{k = 1}^{L} 
    \left( \bS_{0}  \bW_{t_{k-1}} + \bU_{t_{k-1}} \right) \bB^\top_{t_{k-1}}
    \\
    &= \bS_{0} 
    + \sum_{k = 1}^{L} 
    \left( \bS_{0}  [\bW^\top]^{k, \cdot}_\cdot + [\bU^\top]^{k, \cdot}_\cdot \right)[\bB]_{k, \cdot}^\cdot
    = \bS_{0} 
    + 
    \left( \bS_{0}\bW^\top + \bU^\top \right)\bB
    \end{align*}
    as needed. 
    Similarly, from 
    \begin{align*}
    \bW_{t_{k-1}} &= \bA_{t_{k-1}} + \sum_{m = 0}^{k-2} \bW_{t_m} \bB^\top_{t_{m}} \bA_{t_{k-1}}
    = \bA_{t_{k-1}} + \sum_{m = 1}^{k-1} \bW_{t_{m-1}} \bB^\top_{t_{m-1}} \bA_{t_{k-1}}
    \end{align*}
    we get the 3-tensor equation
    \begin{align*}
        [\bW]_{k, \cdot}^\cdot &= [\bA]_{k, \cdot}^\cdot + \sum_{m=1}^{k-1} [\bA]_{k, \cdot}^\cdot [\bB^\top]^{m, \cdot}_\cdot [\bW]_{m, \cdot}^\cdot
        \\
        &= [\bA]_{k, \cdot}^\cdot + \sum_{m=1}^{L} \mathbb{I}(m < k) ~ [\bA]_{k, \cdot}^\cdot [\bB^\top]^{m, \cdot}_\cdot [\bW]_{m, \cdot}^\cdot
        \\
        &= [\bA]_{k, \cdot}^\cdot + \sum_{m=1}^{L} ( [\bM]_{k, \cdot}^{m, \cdot} \odot [\bA]_{k, \cdot}^\cdot [\bB^\top]^{m, \cdot}_\cdot)  [\bW]_{m, \cdot}^\cdot
    \end{align*}
    so that 
    \begin{equation}
        \bW = \bA + (\bM \odot \bA\bB^\top) \bW.
    \end{equation}
    The equation for $\bU$ follows analogously.

    From 
    \begin{equation}
        \bW = \bA + (\bM \odot \bA\bB^\top) \bW, \quad \bU = \tilde\bA + (\mathbf{M} \odot \bA \bB^\top) \bU
    \end{equation}
    immediately follows that 
    \begin{equation}
        (\mathbf{Id} - \bM \odot \bA\bB^\top)\bW = \bA , \quad 
        (\mathbf{Id} - \bM \odot \bA\bB^\top) \bU = \tilde\bA
    \end{equation}
    so that it remains to prove the existence of a left-inverse computable with {$\mathcal{O}(L^3R^3)$} complexity.
    This follows from {Proposition \ref{prop:triangular_inverse}} given that $[\mathbf{Id}]_{t, \cdot}^{s, \cdot} = \delta_s^t ~ Id$ and $[\bM]_{t, \cdot}^{s, \cdot} = \mathbb{I}(s < t)$ thus, 
    \[
    s > t \implies [ \mathbf{Id} - \bM \odot \bA\bB^\top ]_{t, \cdot}^{s, \cdot} = [ \mathbf{Id} ]_{t, \cdot}^{s, \cdot} - [\bM ]_{t, \cdot}^{s, \cdot} \odot [\bA\bB^\top ]_{t, \cdot}^{s, \cdot} = \mathbf{0}
    \]
    \[
    s = t \implies [ \mathbf{Id} - \bM \odot \bA\bB^\top ]_{t, \cdot}^{s, \cdot} = [ \mathbf{Id} ]_{t, \cdot}^{s, \cdot} - [\bM ]_{t, \cdot}^{s, \cdot} \odot [\bA\bB^\top ]_{t, \cdot}^{s, \cdot} = Id \in GL(R).
    \]
\end{proof}

The following proposition is a reformulation of the classical inversion result for block-triangular matrices:
\begin{proposition}[Triangular Tensor Inversion]
\label{prop:triangular_inverse}
    Let $\bC \in \R^{(T \times R) \times (T \times R)}$ be such that 
    \begin{equation}
        s > t \implies [ \bC ]_{t, i}^{s, j} = 0, \quad 
        \forall t. ~ [ \bC ]_{t, \cdot}^{t, \cdot} \in GL(R).
    \end{equation}
    Then the tensor $\mathbf{D} \in \R^{(T \times R) \times (T \times R)}$ with entries
    \begin{equation}
        s > t \implies \mathbf{D}_{t, \cdot}^{s, \cdot} = \mathbf{0} \in \R^{R \times R}
    \end{equation}
    \begin{equation}
        \mathbf{D}_{t, \cdot}^{t, \cdot} = ( \mathbf{C}_{t, \cdot}^{t, \cdot} )^{-1} \in \R^{R \times R}
    \end{equation}
    \begin{equation}
        s < t \implies \mathbf{D}_{t, \cdot}^{s, \cdot} = - 
        ( \mathbf{C}_{t, \cdot}^{t, \cdot} )^{-1} 
        \left( 
        \sum_{r=s}^{t-1} 
        \mathbf{C}_{t, \cdot}^{r, \cdot}
        \mathbf{D}_{r, \cdot}^{s, \cdot} 
        \right) = - 
        \bD_{t, \cdot}^{t, \cdot}
        \left( 
        \sum_{r=s}^{t-1} 
        \mathbf{C}_{t, \cdot}^{r, \cdot}
        \mathbf{D}_{r, \cdot}^{s, \cdot} 
        \right)
        \in \R^{R \times R}
    \end{equation}
    is inverse for composition \emph{i.e.} $\bD \bC = \mathbf{Id}$.
\end{proposition}

\begin{proof}
    Under the identification $\phi: \R^{(T \times R) \times (T \times R)} \to \R^{TR \times TR}$
    given by $[\phi(\bA)]_{i + (k-1)R}^{j + (m-1)R} = [\bA]_{k,i}^{m,j}$ we have 
    \[
    \phi(\bA) \phi(\bB) = \phi(\bA\bB), 
    \quad 
    \phi(\mathbf{Id}) = Id
    \]
    \emph{i.e.} $\phi$ is an isomorphism for tensor composition. Since $\phi(\bC)$ is a block triangular matrix the result follows from the classical inversion formula \cite{Meyer70}.
\end{proof}

\subsection{Signature-Inspired algorithm for DeltaNet}
\label{app:sig_algo_proofs}

Signature methods are a modern class of algorithms inspired by rough path theory which recently became popular in machine learning applications dealing with sequential data \cite{fermanian2023new, cass2024lecture}. These  methods have seen a rapid rise in popularity in recent years have been applied in a variety of contexts including deep learning 
\cite{kidger2019deep, morrill2021neural, cirone2023neural, hoglund2023neural, cirone2024theoretical, issa2024non, barancikova2024sigdiffusions}, kernel methods \cite{salvi2021signature, salvi2021rough, lemercier2021distribution, lemercier2021siggpde, manten2024signature}, quantitative finance \cite{arribas2020sigsdes, salvi2021higher, horvath2023optimal, pannier2024path, cirone2025rough}, information theory \cite{salvi2023structure, shmelev2024sparse}, cybersecurity \cite{cochrane2021sk}, and computational neuroscience \cite{holberg2024exact} among others. In particular, the following algorithm is inspired from signature kernel techniques presented in \cite{salvi2021signature}, where the authors were able to solve a two dimensional partial differential equation (PDE) by refactoring the operations of the corresponding finite difference solver in a diagonal-wise fashion, thus obtaining a parallelisable scheme.

\begin{theorem}
    The solution to the implicit equation 
    \begin{equation}
        \bW_{t_{k}} =  \bA_{t_{k}} + \sum_{m = 0}^{k-1} \bW_{t_m} \bB^\top_{t_{m}} \bA_{t_k},
    \end{equation}
    is given by the diagonal elements $\bW(t_k, t_k)$ of the system
    \begin{equation}
        \bW(t_0, t_k) = \bA_{t_k}, \quad \bW(t_{k+1},t_{k+1}) = \bW(t_k,t_{k+1}) + \bW(t_k,t_k) \bB_{t_k}^\top \bA_{t_{k+1}}
    \end{equation}
    \begin{equation}
        m < k \implies \bW(t_{m+1}, t_{k+1}) =  \bW(t_{m},t_{k+1}) + \bW(t_{m+1},t_{k}) - \bW(t_{m},t_{k}) + \bW(t_{m},t_{m}) 
        \bB^\top_{t_{m}} ( \bA_{t_{k+1}} - \bA_{t_k}).
    \end{equation}
\end{theorem}

\begin{proof}
    We want to show how 
    \begin{equation}
        \bW(t_{k}, t_{k}) =  \bA_{t_{k}} + \sum_{m = 0}^{k-1} \bW(t_m, t_m) \bB^\top_{t_{m}} \bA_{t_k},
    \end{equation}
    to do so we will prove the stronger statement for $n \leq k$
    \begin{equation}
        \bW(t_{n}, t_{k}) =  \bA_{t_{k}} + \sum_{m = 0}^{n-1} \bW(t_m, t_m) \bB^\top_{t_{m}} \bA_{t_k}.
    \end{equation}

    Let's reason by induction on $k$. The base case $k=0$ is trivially true.
    Assume then the equality holds until $k > 0$ included. 
    Reasoning by induction on $n  \leq k+1$ we have that the base case $n = 0$ once again is trivial, moreover for $n < k$
    \begin{align*}
        \bW(t_{n+1}, t_{k+1})
        =&
        ~  \bW(t_{n},t_{k+1}) + \bW(t_{n+1},t_{k}) - \bW(t_{n},t_{k}) + \bW(t_{n},t_{n}) 
        \bB^\top_{t_{n}} ( \bA_{t_{k+1}} - \bA_{t_k})
        \\
        =& 
        ~ \bA_{t_{k+1}} + 
        \sum_{m = 0}^{n-1} \bW(t_m, t_m) \bB^\top_{t_{m}} \bA_{t_{k+1}} 
        + \bA_{t_{k}} + \sum_{m = 0}^{n} \bW(t_m, t_m) \bB^\top_{t_{m}} \bA_{t_{k}}
        \\
        & - \bA_{t_{k}} - \sum_{m = 0}^{n-1} \bW(t_m, t_m) \bB^\top_{t_{m}} \bA_{t_{k}}
        + \bW(t_{n},t_{n}) 
        \bB^\top_{t_{n}} ( \bA_{t_{k+1}} - \bA_{t_k})
        \\
        =& 
        ~ \bA_{t_{k+1}} + 
        \sum_{m = 0}^{n-1} \bW(t_m, t_m) \bB^\top_{t_{m}} \bA_{t_{k+1}} 
        + \bW(t_n, t_n) \bB^\top_{t_{n}} \bA_{t_{k}}
        \\
        & 
        + \bW(t_{n},t_{n}) 
        \bB^\top_{t_{n}} ( \bA_{t_{k+1}} - \bA_{t_k})
        \\
        =& 
        ~ \bA_{t_{k+1}} + 
        \sum_{m = 0}^{n-1} \bW(t_m, t_m) \bB^\top_{t_{m}} \bA_{t_{k+1}} 
        + \bW(t_{n},t_{n}) 
        \bB^\top_{t_{n}} \bA_{t_{k+1}}
        \\
        =& ~ \bA_{t_{k+1}} + 
        \sum_{m = 0}^{n} \bW(t_m, t_m) \bB^\top_{t_{m}} \bA_{t_{k+1}} 
    \end{align*}
    as needed. To conclude note that if the equality holds for all $n < k$ then it does also for $n=k$ since
    \begin{align*}
        \bW(t_{k+1},t_{k+1}) &= \bW(t_k,t_{k+1}) + \bW(t_k,t_k) \bB_{t_k}^\top \bA_{t_{k+1}}
        \\ 
        &= 
        \bA_{t_{k+1}} + \sum_{m = 0}^{k-1} \bW(t_m, t_m) \bB^\top_{t_{m}} \bA_{t_{k+1}} + \bW(t_k,t_k) \bB_{t_k}^\top \bA_{t_{k+1}}
        \\ 
        &= 
        \bA_{t_{k+1}} + \sum_{m = 0}^{k} \bW(t_m, t_m) \bB^\top_{t_{m}} \bA_{t_{k+1}}.
    \end{align*}
\end{proof}

\newpage
\section{Pseudocode}

% Reference: \cite{yang2024parallelizinglineartransformersdelta}

\lstset{
    language=Python,
    basicstyle=\ttfamily\small,
    numbers=left,
    numberstyle=\ttfamily\textcolor[rgb]{0.5,0.5,0.5}{\scriptsize},
    numbersep=8pt,
    frame=lines,
    framesep=2mm,
    breaklines=true,
    keywordstyle=\color{blue},
    commentstyle=\color[rgb]{0,0.6,0}{\itshape},
    stringstyle=\color{red},
    showstringspaces=false,
    columns=flexible,
    mathescape=true
}

\begin{figure}[h!]
\centering
\begin{lstlisting}
def chunk_tensorInv_deltaRule(A, Alpha, B, S_0):
    '''
    Computes S_1 on the chunk [0, 1]
    
    A/Alpha/B: tensors shape [L, R, d]
    S_0: tensor of shape [d, d]
    L: length, R: rank, d: dimension 

    Naive Complexity: O(L^2 * R * (d^2 + R*d + L*R^2)))  
    Parallel Complexity: O_par(L^2*R + d)
    Memory: O(L^2*R^2 + L*R*d + d^2)
    '''
    
    L, R, d = A.shape
    # Utils || O(L^2 * R^2) || O_par(1)
    M = torch.tril(torch.ones((L, L))) - torch.eye(L)
    Id = torch.eye(R)[None, :, None, :] * torch.eye(L)[:, None, :, None]
    
    # Compute C := (I - M \odot (A @ B.T)) || O(L^2 * R^2 * d) || O_par(d)
    C = einsum(A, B, "t i k, s j k -> t i s j") # [L, R, L, R]
    C = Id - M[:, None, :, None] * C
    
    # Compute the inverse D := C^{-1} || O(L^3 * R^3) || O_par(L^2 * R)
    D = torch.zeros_like(C)
    for t in range(L):
        D[t, :, t, :] = torch.eye(R)
        # Forward substitute || O(L^2 * R^3) || O_par(L * R)
        D[t, :, :t, :] -= einsum(C[t, :, :, :], D[:, :, :t, :], "i r j, r j s k -> i s k") 
    
    # Compute W and U || O(L^2 * R^2 * d) || O_par(L * R)
    W = einsum(D, A, "t i s j, s j k -> t i k")
    U = einsum(D, Alpha, "t i s j, s j k -> t i k")

    # Compute the final value S_1 || O(L * R * d^2) || O_par(d + L * R)
    # S_1 = S_0 + ( U + (W @ S_0.T)).T @ B 
    temp = U + einsum(W, S_0, "t i k, k m -> t i m")
    S_1 = S_0 + einsum(temp, B, "t i m, t i k -> m k")
    
    return S_1
\end{lstlisting}
\caption{Pytorch code showing a simple implementation of the algorithm.}
\label{code_tensorInv}
\end{figure}

\newpage
% Reference: \cite{yang2024parallelizinglineartransformersdelta}

\lstset{
    language=Python,
    basicstyle=\ttfamily\scriptsize,
    numbers=left,
    numberstyle=\ttfamily\textcolor[rgb]{0.5,0.5,0.5}{\scriptsize},
    numbersep=8pt,
    frame=lines,
    framesep=2mm,
    breaklines=true,
    keywordstyle=\color{blue},
    commentstyle=\color[rgb]{0,0.6,0}{\itshape},
    stringstyle=\color{red},
    showstringspaces=false,
    columns=flexible,
    mathescape=true
}

\begin{figure}[!ht]
{
\centering
\begin{lstlisting}
def chunk_sigDeltaRule_optim(A, Alpha, B, S_0):
    '''
    Computes S_1 on the chunk [0, 1] with signature kernel method
    A/Alpha/B: tensors shape [L, R, d]
    S_0: tensor of shape [d, d]
    L: length, R: rank, d: dimension 

    Naive complexity: O(L^2 * R * (d^2 + R*d)))
    Parallel Complexity: O_par(L*R + d)
    Memory: O(L^2*R^2 + L*R*d + d^2)
    '''
    
    L, R, d = A.shape       
    idx = torch.arange(0, L)

    # G[s, t] = (A[s+1] - A[s]) @ B[t].T  || [L, L, R, R] || O(L^2 * R^2 * d) || O_par(d)
    G = einsum(torch.roll(A, -1, dims=0) - A, B, 's i k, t j k -> s t i j')
    # G_off_diag[t] = A[t+1] @ B[t].T || [L, R, R] || O(L * R^2 * d) || O_par(d)
    G_off_diag = einsum(torch.roll(A, -1, dims=0), B, 't i k, t j k -> t i j')

    # Initialise W and U chunk 
    A_Alpha = torch.cat([A[..., None], Alpha[..., None]], dim=-1) 
    WU = torch.where(idx[:, None, None, None] == 0, A_Alpha, 0) # [L, R, d, 2]

    diagonals = torch.zeros([3, L, R, d, 2])
    for i in range(1, 2*L - 1): # Multiply all O(---) inside by O(L)
    
        # Roll Diags || O( L * R * d ) || O_par(1)
        diagonals = torch.roll(diagonals, -1, dims=0)

        # Utils
        t = i // 2
        mask = (idx <= t) & (idx > i - L) # [L]

        # Compute gram entries on diagonal || O( L * R^2 ) || O_par(1)
        G_entries = torch.zeros([L, R, R])
        for s in range(max(0, i - L - 1), min(i-1, L)):
            G_entries[s] = G[i - s - 2][s]

        # Initialise next diagonal
        if i < L: diagonals[1, 0] = A_Alpha[i]

        # Diagonals Update || O(L * R^2 * d) || O_par(R)
        temp = diagonals[0] + torch.roll(diagonals[0], -1, dims=0) - diagonals[-1]  
        temp += einsum(G_entries, WU, 't i j, t j k u -> t i k u')
        temp = torch.roll(temp, 1, dims=0)
        diagonals[1, 1:] = temp[1:]
        diagonals[1] *= mask[:, None, None, None]

        # Set WU value || O(R^2 * d) || O_par(R)
        if i % 2 == 0:
            temp = diagonals[0, t-1].clone()
            temp += einsum(G_off_diag[t-1], WU[t-1], 'i j, j k u -> i k u')
            diagonals[1, t], WU[t] = temp, temp

    W, U = WU[..., 0], WU[..., 1]

    # Compute the final value S_1 || O(L * R * d^2) || O_par(d + L * R)
    # S_1 = S_0 + ( U + (W @ S_0.T)).T @ B 
    temp = U + einsum(W, S_0, "t i k, k m -> t i m")
    S_1 = S_0 + einsum(temp, B, "t i m, t i k -> m k")
    
    return S_1
\end{lstlisting}
}
\caption{Pytorch code showing a simple implementation of the sigDelta algorithm.}
\label{sigDelta_code_optimal}
\end{figure}

\newpage

\end{document}